\documentclass[11pt]{article}

\usepackage[a4paper,margin=0.8in]{geometry}
\usepackage{amsthm,mathtools,amssymb}
\usepackage[numbers,sort&compress]{natbib}
\usepackage{tabularx}
\usepackage{booktabs}
\usepackage{graphicx}
\usepackage{microtype}
\usepackage{xurl}
\usepackage{hyperref}

\DeclareRobustCommand{\CE}{\ifmmode\mathrm{CE}\else\textnormal{CE}\fi}
\DeclareRobustCommand{\SOL}{\ifmmode\mathrm{SOL}\else\textnormal{SOL}\fi}
\DeclareRobustCommand{\WSOL}{\ifmmode\mathrm{wSOL}\else\textnormal{wSOL}\fi}

\newcommand{\E}{\mathbb{E}}
\newcommand{\ind}{\mathbf{1}}

\newtheorem{definition}{Definition}
\newtheorem{proposition}{Proposition}

\title{Weighted Score-Oriented Losses for\\
Temporally Localized Event Prediction}

\author{
Edoardo Legnaro\thanks{Corresponding author. Email: \texttt{edoardo.legnaro@edu.unige.it}}\\
Department of Mathematics, University of Genoa\\
Via Dodecaneso 35, 16146 Genoa, Italy
\and
Sabrina Guastavino\\
Department of Mathematics, University of Genoa\\
Via Dodecaneso 35, 16146 Genoa, Italy\\
\texttt{sabrina.guastavino@unige.it}
\and
Francesco Marchetti\\
Department of Mathematics ``Tullio Levi-Civita'', University of Padua\\
Via Trieste 63, 35121 Padua, Italy\\
\texttt{francesco.marchetti@math.unipd.it}
}

\date{}

\begin{document}
\maketitle

\begin{abstract}
Operational event-detection systems are rarely assessed by pointwise accuracy alone.
In anomaly detection, changepoint detection, and warning systems, the utility of an
alarm depends on its temporal position relative to an event. This produces a
score--loss mismatch. Neural networks are commonly trained with classical loss
functions, such as cross-entropy, whereas deployment decisions are obtained by
thresholding network predictions, merging alarms through post-processing rules,
and evaluating them with event-based metrics defined by detection windows and
false-alarm costs.
This paper studies a temporally localized specialization of weighted
score-oriented loss (\WSOL) for event prediction.  Starting from score-oriented
losses based on expected confusion matrices and from the weighted \SOL\
framework of \cite{Marchetti2024}, we
consider temporal weights that discount near-event false positives
and reduce false-negative penalties when an event is preceded by an admissible
alarm. The resulting objective is differentiable with respect to the network
predictions, and therefore can be optimized by back-propagation. It can be
instantiated with balanced accuracy, true skill statistic, F1,
critical success index, and related confusion-matrix scores. 
We evaluate the proposed approach by comparing cross-entropy, unweighted
score-oriented loss, and \WSOL\ on three benchmark datasets for time-series
event prediction and detection.
The results show that \WSOL\ can improve performance when the evaluation utility is localized in time and is not already encoded by the pointwise labels.
\end{abstract}


\section{Introduction}
Binary neural classifiers are usually optimized by minimizing a decomposable loss,
most commonly cross-entropy. Here, decomposable means that the training objective
is a sum or average of terms, each depending only on the local
prediction and label at one index. The final system, however, is often evaluated
through non-decomposable performance metrics computed after thresholding,
where the metric value depends jointly on
multiple predictions through global confusion-matrix counts and, in event-based settings, on additional post-processing rules that couple neighboring predictions.
In event-based time-series applications, this
discrepancy is particularly severe. A detector may be credited for firing shortly
before an anomalous window, penalized for repeated alarms, or judged according to a
range-level score, not according to the number of correctly classified time
points. Two prediction batches with identical pointwise confusion matrices may thus have very different operational value.

The mismatch between a training loss and an evaluation metric has been studied in
several forms. Classical cost-sensitive learning introduces asymmetric penalties
for false positives and false negatives \cite{Elkan2001}. This handles unequal
error costs but does not by itself model the position of an error in a temporal
sequence. Losses designed for class imbalance, including focal loss
\cite{Lin2017} and class-balanced reweighting \cite{Cui2019}, similarly modify
the relative contribution of samples but remain pointwise objectives.

Another line of work studies direct optimization of non-decomposable metrics.
Consistency and surrogate-regret analyses for F-like and other generalized
classification scores were developed in
\cite{Narasimhan2014NonDecomposable,Kar2014OnlineNonDecomposable,Kotlowski2016SurrogateRegret,Bao2020LinearFractional},
while scalable differentiable formulations include \cite{Eban2017}. A related
overlap-based relaxation is the Lovasz-Softmax loss \cite{Berman2018}.
These methods share the goal of aligning training with evaluation, but they are
usually formulated for independent binary classification or spatial overlap and
do not encode localized temporal utility around an event boundary.

In early warning and early event prediction, temporal structure has been
incorporated through precision-constrained learning
\cite{Rath2022MinPrecision}, temporal label shaping
\cite{Yeche2023TemporalLabelSmoothing}, and survival-style alarm policies
\cite{Yeche2024DSA}. Time-series anomaly detection has likewise developed
increasingly event-aware evaluation protocols, including Numenta Anomaly Benchmark (NAB)
\cite{Ahmad2017NAB,NABZenodo2017}, range-based precision and recall
\cite{Tatbul2018PrecisionRecall}, local evaluation
\cite{Huet2022LocalEvaluation}, threshold-independent range measures such as
Volume under the surface (VUS) \cite{Paparrizos2022VUS}, and broader taxonomies of metric behavior
\cite{Sorbo2024MetricMaze}; see also \cite{Doshi2022RewardOnce} for a critique
of inflated adjusted event scoring. These works sharpen the evaluation side of
the problem, but they do not turn the temporal reward
structure itself into a differentiable training objective.

Weighted event losses have also appeared in adjacent detection domains.
Examples include class-imbalance and multitask weighting for rare audio event
detection \cite{Phan2018RareAudioWeighted} and onset/offset-weighted
cross-entropy for sound event detection \cite{Song2024OnsetOffsetWeighted}.
Such losses inject temporal emphasis through pointwise reweighting, but they
are not derived from a confusion-matrix score family.

Score-oriented losses (\SOL) were introduced to reduce the mismatch between
training objectives and threshold-based evaluation metrics by replacing the
discontinuous hard confusion matrix with its expectation under a threshold
distribution \cite{Marchetti2022SOL}. In this way, scores such as balanced
accuracy, F1, true skill statistic, and critical success index can be turned into
differentiable loss functions. This idea was later extended in
\cite{Marchetti2024} through a general framework for weighted classification
metrics, showing that weighted score-oriented losses encompass cost-sensitive
learning, weighted cross-entropy, and value-weighted skill scores
\cite{Guastavino2021ValueWeighted}. 

The present work builds on this framework by deriving and testing a temporally
weighted score-oriented loss (\WSOL) for event detectors whose downstream
utility is concentrated in a reward window around an event.  The contribution is
a focused specialization to temporally localized event prediction: localized
event utility is written directly in expected confusion-matrix terms,
reformulated in event-detection language, and assessed empirically on anomaly-
and changepoint-detection benchmarks.  This is the
setting in which the loss contains information that is absent from ordinary
pointwise labels. 

Indeed, time-series anomaly detection is commonly evaluated with event- or range-aware
metrics rather than pure pointwise accuracy.  The NAB
uses application profiles that reward early detections inside labeled anomaly
windows and penalize false alarms \cite{Ahmad2017NAB,NABZenodo2017}.  More
general range-based precision and recall measures were proposed by Tatbul et
al. \cite{Tatbul2018PrecisionRecall}.  Value-weighted skill scores similarly
distinguish temporally useful errors from harmful errors in forecasting problems
\cite{Guastavino2021ValueWeighted,guastavino2022prediction}. These evaluation protocols motivate a
training loss that can see the temporal reward structure during optimization
instead of only after threshold tuning.

The paper is organized as follows.  Section~\ref{sec:wsol} formalizes the
temporally weighted score-oriented loss: starting from the expected
confusion-matrix framework of \SOL, we introduce the future-event proximity
weight and the prior-alarm correction terms, define the resulting \WSOL\
objective, and verify that it reduces to \SOL\ when no temporal weighting is
applied.  Section~\ref{sec:empirical} describes the experimental setup, including the
three benchmark datasets (NAB, SKAB, Exathlon), the common Temporal Convolutional Network (TCN) architecture,
and the training and evaluation pipeline.  Section~\ref{sec:results} reports
the held-out comparisons.  Section~\ref{sec:discussion} interprets the results
across datasets, analyzes the conditions under which temporal weighting helps,
and discusses the role of temporal weight selection and the limitations of the current
formulation.  Section~\ref{sec:conclusion} summarizes the findings and outlines
directions for future work.
The implementation of \WSOL\ is available at
\url{https://github.com/edoardolegnaro/ScoreOrientedLosses.git}.

\section{Weighted Score-Oriented Loss}
\label{sec:wsol}
Let \(\{(x_i,y_i)\}_{i=1}^n\) be a time-ordered binary dataset of \(n\) samples, with input sample \(x_i\in\Omega\subset \mathbb{R}^d\) and
\(y_i\in\{0,1\}\). A neural network with parameters \(\theta\), denoted by $f_{\theta}$, produces a
prediction
$p_i=f_\theta(x_i)\in(0,1)$.  A threshold \(\tau\in(0,1)\) gives the hard
prediction
$\hat y_i(\tau)=\ind\{p_i\geq \tau\},$
where $\ind\{\cdot\}$ denotes the indicator function.
The classical confusion-matrix entries are
\begin{align}
  \mathrm{TP}_\tau &= \sum_{i=1}^n y_i \ind\{p_i\geq \tau\}, &
  \mathrm{FN}_\tau &= \sum_{i=1}^n y_i \ind\{p_i< \tau\}, \\
  \mathrm{FP}_\tau &= \sum_{i=1}^n (1-y_i)\ind\{p_i\geq \tau\}, &
  \mathrm{TN}_\tau &= \sum_{i=1}^n (1-y_i)\ind\{p_i< \tau\}.
\end{align}
A given score can be written as
$s_\tau = s(\mathrm{TN}_\tau,\mathrm{FP}_\tau,\mathrm{FN}_\tau,\mathrm{TP}_\tau),$ where $s$ denotes a classification performance or skill score defined in terms of the confusion matrix.
Examples used in the experiments include balanced accuracy, which ranges in \([0,1]\) and attains its optimum value at \(1\), and the true skill statistic, which ranges in \([-1,1]\) and is optimal when it reaches \(1\):
\begin{equation}
    \mathrm{BA}
  = \frac{1}{2}\left(
    \frac{\mathrm{TP}}{\mathrm{TP}+\mathrm{FN}}
    +
    \frac{\mathrm{TN}}{\mathrm{TN}+\mathrm{FP}}
  \right), \quad   \mathrm{TSS}
  = \frac{\mathrm{TP}}{\mathrm{TP}+\mathrm{FN}}
  +
  \frac{\mathrm{TN}}{\mathrm{TN}+\mathrm{FP}}
  - 1.
\end{equation}

These skill scores are particularly suitable for imbalanced classification problems, and they are directly related through \(\mathrm{BA} = \frac{1}{2}(\mathrm{TSS} + 1)\).
The score \(s_\tau\) is discontinuous with respect to \(p_i\), and therefore in the neural-network parameters.

We replace the fixed threshold with a random threshold variable $T$ with cumulative distribution function $F$ on $(0,1)$.
Let $z_i = F(p_i).$
Since \(\mathbb{P}(T\leq p_i)=F(p_i)\), the expected confusion-matrix entries are
\begin{align}
  \mathrm{TP}_F &= \E_T[\mathrm{TP}_T] = \sum_{i=1}^n y_i z_i, &
  \mathrm{FN}_F &= \E_T[\mathrm{FN}_T] = \sum_{i=1}^n y_i(1-z_i), \\
  \mathrm{FP}_F &= \E_T[\mathrm{FP}_T] = \sum_{i=1}^n (1-y_i)z_i, &
  \mathrm{TN}_F &= \E_T[\mathrm{TN}_T] = \sum_{i=1}^n (1-y_i)(1-z_i).
\end{align}
For a score \(s\), the score-oriented loss is
\begin{equation}
  \mathcal{L}_{\SOL}(\theta)= 1 - 
  s\!\left(\mathrm{TN}_F,\mathrm{FP}_F,\mathrm{FN}_F,\mathrm{TP}_F\right).
  \label{eq:sol}
\end{equation}
When \(F\) is smooth, \(\mathcal{L}_{\SOL}\) is differentiable wherever the
chosen score is differentiable.

Event-detection metrics often reward alarms in a neighborhood of an event, not only at positively labeled time points.  Let \(H\) be a temporal horizon and let
\(\omega=(\omega_1,\ldots,\omega_H)\) be a non-negative utility weight vector, with \(\omega_h\in [0,1)\) for each \(h\in\{1,\dots,H\}\).  In the
one-sided formulation below, \(h\) indexes the distance from an alarm to a
future event, or equivalently the lag of a prior alarm relative to an event.
Larger \(\omega_h\) means that an alarm at that lag is more useful.
In the experiments, \(\omega_h\) is non-increasing and concentrated near the
event boundary.  This construction can be read as a temporally localized,
event-prediction specialization of the chronological value-weighted \WSOL\
setting already formalized in \cite{Marchetti2024}, rewritten here in direct
event-utility notation.

Two effects are encoded.  First, a positive prediction at a negative time point should not be penalized as a full false alarm if an event occurs shortly afterward.  Define the future-event proximity
\begin{equation}
  a_i^+ = \max_{1\leq h\leq H}\omega_h y_{i+h},
  \label{eq:future-proximity}
\end{equation}
where \(y_{j}=0\) when the index \(j\) falls outside the range of the considered sequence.  The weighted expected false-positive term is
\begin{equation}
  \mathrm{FP}_{F, \omega}
  =
  \sum_{i=1}^n (1-y_i)(1-a_i^+)z_i.
  \label{eq:wfp}
\end{equation}
Thus an alarm immediately before an event is discounted when the target metric
would treat it as useful.

Second, an event should not be treated as fully missed if the model has already
raised a strong nearby alarm.  Let \(z_{i-h}=0\) for \(i-h<1\).  A product-style
temporal correction is
\begin{equation}
  c_i^{\mathrm{prod}}
  =
  \sum_{h=1}^H \omega_h [z_{i-h}-z_i]_+,
  \label{eq:prod-correction}
\end{equation}
where \([u]_+=\max(u,0)\).  A max-style correction, used in the main experiments,
keeps only the upper envelope of previous alarms:
\begin{equation}
  c_i^{\mathrm{max}}
  =
  \sum_{h\in \mathcal{R}_i}
  (\omega_h-\omega_{h^+})[z_{i-h}-z_i]_+ .
  \label{eq:max-correction}
\end{equation}
The set \(\mathcal{R}_i\) is defined by scanning backward from time \(i\).  A lag
\(h\) is included in \(\mathcal{R}_i\) only if the soft alarm value \(z_{i-h}\) is
strictly larger than all soft alarm values at closer lags \(1,\ldots,h-1\).  Thus
\(\mathcal{R}_i\) keeps only the successive record-high previous alarms: weaker
alarms that are dominated by a closer alarm cannot add extra credit for the same
event.  If \(\mathcal{R}_i=\{h_1<h_2<\cdots<h_m\}\), then \(h_j^+=h_{j+1}\) for
\(j<m\), and for the last retained lag we set \(\omega_{h_m^+}=0\).  The factor
\(\omega_{h_j}-\omega_{h_j^+}\) assigns only the incremental temporal weight
between two consecutive record alarms.  This running-maximum decomposition is a
piecewise differentiable surrogate for giving credit to the strongest useful prior alarm
without repeatedly rewarding the same event.

The weighted expected false-negative term is then
\begin{equation}
  \mathrm{FN}_{F, \omega}
  =
  \sum_{i=1}^n y_i \left(1-z_i-c_i\right),
  \label{eq:wfn}
\end{equation}
where \(c_i\) is either \eqref{eq:prod-correction} or \eqref{eq:max-correction}.
The true-positive and true-negative terms remain
\[
  \mathrm{TP}_{F,\omega}=\sum_{i=1}^n y_i z_i,\qquad
  \mathrm{TN}_{F,\omega}=\sum_{i=1}^n (1-y_i)(1-z_i).
\]

\begin{definition}[Temporal weighted score-oriented loss]
Given a confusion-matrix score \(s\), a cumulative distribution function \(F\) for a random threshold \(T\), a horizon
\(H\), and a temporal weight vector \(\omega\), the temporal weighted score-oriented
loss is
\begin{equation}
  \mathcal{L}_{\WSOL}(\theta)
  =
  1-
  s\!\left(
    \mathrm{TN}_{F,\omega},
    \mathrm{FP}_{F,\omega},
    \mathrm{FN}_{F,\omega},
    \mathrm{TP}_{F,\omega}
  \right).
  \label{eq:wsol}
\end{equation}
\end{definition}

When the temporal weights are zero (or no horizon is used), the weighted loss reduces to the standard score-oriented loss.

\begin{proposition}
If \(H=0\) or \(\omega_h=0\) for all \(h\), then
\(\mathcal{L}_{\WSOL}=\mathcal{L}_{\SOL}\).
\end{proposition}

\begin{proof}
With zero temporal weights, \(a_i^+=0\) and \(c_i=0\) for all \(i\).  Equations
\eqref{eq:wfp} and \eqref{eq:wfn} reduce to the unweighted expected
false-positive and false-negative terms, while the true-positive and
true-negative terms are unchanged.  Substitution into \eqref{eq:wsol} gives
\eqref{eq:sol}.
\end{proof}

Temporal weighting supplies information that is not present in ordinary pointwise supervision.  It should therefore be useful when the evaluation metric rewards alarms in a localized window around an event and the pointwise labels do not already encode that temporal structure.
If positive labels already cover the full region in which alarms are useful, then cross-entropy receives a dense and aligned training signal.  If anomalous labels span long disturbance ranges, a localized temporal weight vector may even conflict with the range-level evaluation metric.
In the following section we empirically test this distinction.

\section{Empirical Validation}
\label{sec:empirical}

In this section, we evaluate the proposed \WSOL\ loss \eqref{eq:wsol} against
cross-entropy and the unweighted \SOL\ loss \eqref{eq:sol}.  The comparison uses
three time-series benchmarks that differ in whether their evaluation utility is
temporally localized.

All experiments compare three losses under the same model architecture and data pipeline within each dataset.
For \SOL\ and \WSOL, the internal score family is selected from balanced accuracy and TSS, with a uniform threshold distribution.
For \WSOL, the temporal weight vector and horizon are selected from a finite
set of candidate families fixed before model selection.  All final comparisons
use validation-selected loss configurations, thresholds, and post-processing
parameters, and the selected pipeline is then evaluated once on the held-out split.
For NAB, these candidates are specified from the official positive utility
branch, along with one short-horizon control.
Specifically, we construct temporal weight vectors that assign higher weights to
predictions close to an event and progressively lower weights at larger temporal
distances.  This approximates the utility structure induced by the NAB scoring
function.
Let \(S(r)=2\sigma(-5r)-1\) denote the scaled sigmoid function underlying the
positive branch of the NAB utility profile, where \(\sigma(\cdot)\) is the
standard sigmoid function.
Here, \(r \in [-1,1]\) is a normalized coordinate along that branch, ordered
from the highest-reward side (\(r=-1\)) to the low-reward side (\(r=1\)).
Let \(r_h=-1+2(h-1)/(H-1)\) define a discretization of this normalized temporal
window, so \(h=1\) receives the largest temporal credit. We define
\begin{equation}\label{eq:u_h}
  u_h = \frac{\max(S(r_h),0)}{S(-1)},
\end{equation}
which represents a normalized and non-negative approximation of the NAB utility profile, rescaled such that its maximum value equals one.
The NAB-shaped temporal weight vectors are then defined as
\begin{equation}
\omega_h = b_H + (a_H - b_H) u_h^{\gamma_H},
\end{equation}
where \(H \in \{8,16,32,64\}\) controls the temporal resolution of the weight
vector, and the parameters \((a_H, b_H, \gamma_H)\) determine its amplitude,
baseline, and sharpness.  In the experiments, the following horizon-specific
parameter triples are considered:
\[
(a_H,b_H,\gamma_H)\in
\{(0.55,0.04,2.5),(0.50,0.025,3.0),(0.42,0.012,4.0),(0.34,0.006,5.0)\}.
\]
In addition, NAB includes a short-horizon control with \(H=4\) and
\(\omega=(0.45,0.20,0.10,0.05)\).  These functional forms and finite candidate
sets are specified a priori; validation selects among them but does not learn
the temporal weights continuously.

SKAB uses the same one-sided decaying family of temporal weight vectors with
\(H\in\{8,16,32,64\}\) to match its local reward window.  Exathlon instead uses
front-loaded heuristic temporal weight vectors with
\(H\in\{8,16,32,64,128\}\), because its scoring function does not provide an
analogous localized one-sided reward function.  These vectors have the form
\[
  \omega_h=t_H+(p_H-t_H)
  \left(1-\frac{h-1}{H-1}\right)^{q_H},
\]
with horizon-specific parameters
\[
\begin{aligned}
(p_8,t_8,q_8)&=(0.58,0.12,1.6), &
(p_{16},t_{16},q_{16})&=(0.54,0.06,2.4),\\
(p_{32},t_{32},q_{32})&=(0.50,0.025,3.2), &
(p_{64},t_{64},q_{64})&=(0.44,0.010,4.2),\\
(p_{128},t_{128},q_{128})&=(0.38,0.004,5.4).
\end{aligned}
\]
Specifically, Exathlon uses the benchmark's range-based anomaly-detection (AD)
scores, which are F-scores built from range precision and range recall.  These
scores use increasingly restrictive definitions of what it means to cover an
anomalous interval: AD2 uses flat overlap credit over each labeled anomaly
range, AD3 adds a front-position bias that rewards earlier overlap within the
range, and AD4 further removes duplicate credit for repeatedly matching the same
range.  We use AD4 as the primary selection metric.

Table~\ref{tab:datasets} summarizes the benchmarks: NAB
\cite{Ahmad2017NAB,NABZenodo2017}, SKAB \cite{SKAB2020}, and Exathlon
\cite{Jacob2021Exathlon}.  NAB and SKAB are the primary empirical datasets because their evaluation functions reward localized alarms.
Exathlon is included as a boundary case: it is also an anomaly-detection
benchmark, but its labels span long disturbance ranges.  It therefore tests
whether the proposed loss helps beyond the mere presence of temporal structure.

\begin{table}[hbp]
\centering
\small
\begin{tabularx}{\textwidth}{p{1.5 cm}p{2 cm}X X}
\toprule
Dataset & Domain & Structure & Class imbalance \\
\midrule
NAB & Streaming anomaly detection & 58 univariate time series with labeled anomaly windows and application-specific scoring profiles. & Window-membership positives: 33{,}495 / 365{,}558 (9.16\%). \\
SKAB & Industrial changepoint detection & 35 multivariate sensor files with physical channels and changepoint/anomaly labels. & Changepoint positives: 129 / 37{,}401 (0.345\%). \\
Exathlon & Spark application anomaly detection & 81 readable traces after archive repair, 19 derived features, and 86 labeled anomaly ranges across 6 disturbance types. & Range-labeled: 86 anomaly ranges over 81 traces (range-level, not pointwise prevalence). \\
\bottomrule
\end{tabularx}
\caption{Datasets used in the validation.  The paper focuses on datasets whose
metrics expose temporal label--utility mismatch, with Exathlon retained as a
boundary case.}
\label{tab:datasets}
\end{table}

The imbalance profile differs sharply across benchmarks and helps interpret the
loss comparisons.  NAB has a moderate pointwise positive prevalence
(33{,}495/365{,}558, 9.16\%), whereas SKAB is strongly sparse at the point level
(129/37{,}401, 0.345\%), making changepoint positives much rarer during
optimization and thresholding.  Exathlon is reported differently: its labels are
annotated as anomaly ranges, and each trace is one full multivariate recording run.
Thus, 86 ranges over 81 traces means that one run can contain zero, one, or multiple labeled anomaly intervals; this is a range-level descriptor of event annotation density rather than pointwise class prevalence.
For this reason, the Exathlon imbalance entry is not numerically comparable to the NAB and SKAB percentages, and it should be interpreted in terms of interval coverage and range-based evaluation.
After the 15-second resampling used in our pipeline, Exathlon contains 5{,}704 anomalous samples out of 126{,}674 total samples, corresponding to 4.50\% anomaly coverage; equivalently, roughly 95.5\% of sampled trace time is non-anomalous.

Each benchmark uses a residual temporal convolutional network (TCN) with causal
one-dimensional convolutions, ReLU nonlinearities, dropout, and a final sigmoid
head.  Each residual block contains two causal convolutional layers with kernel
size 3 and dilation \(2^b\) in block \(b\), followed by a residual projection
when the channel dimension changes.  The dropout rate is 0.1 after each
convolution and after each hidden layer in the dense head.  NAB uses six
residual TCN blocks with 48 channels, sequence length 96, batch size 4, and
dense head widths 32 and 8.  SKAB and Exathlon use five residual TCN blocks with
32 channels, sequence length 120, batch size 2, and dense head widths 24 and 8.
All models are trained in PyTorch with Adam, learning rate \(10^{-4}\), mixed
bfloat16 precision, maximum 80 epochs for NAB and 60 epochs for SKAB and
Exathlon, and early stopping with patience 8.  Early stopping and checkpoint
selection maximize the validation metric used by the corresponding benchmark:
the standard NAB profile for NAB and SKAB, and AD4 for Exathlon, after
validation-set threshold and post-processing selection.
The same backbone configuration is used for \CE, \SOL, and \WSOL\ within each dataset.

For event metrics, the final predicted raw probabilities are converted to binary alarms using the threshold tuned on the validation set.
NAB and SKAB use NAB-style scoring profiles and local-maximum post-processing to evaluate event detection quality.
Local-maximum post-processing identifies peaks in the probability sequence: a predicted alarm is retained if its predicted probability is higher than its immediate neighbors, which reduces redundant nearby detections.
SKAB applies an additional refractory non-maximum-suppression window of 30 time steps: after a prediction peak is detected, any other peaks within the following 30 time steps are suppressed to avoid multiple alarms for the same underlying event.
Each dataset is evaluated across four outer folds and five random seeds, giving 20 split--seed comparisons per dataset.  The \CE\ baseline contributes one run per split--seed pair (20 total); \SOL\ is trained under two score-oriented configurations, yielding 40 runs.  \WSOL\ is trained once per temporal candidate family per split--seed pair: NAB and Exathlon use 10 candidate families (200 runs each), while SKAB uses 8 (160 runs).

\section{Results}
\label{sec:results}

Table~\ref{tab:headline} reports validation-selected held-out test performance, summarized as the mean and standard error of the evaluation metric over the 20 split--seed runs for each dataset.  

The two datasets whose utility is localized around event windows, NAB and SKAB, show positive mean \WSOL\ gains over \CE.  Exathlon shows the opposite behavior: although some individual weighted candidates improve over \SOL, validation-selected \WSOL\ remains well below \CE\ on the range-level objective.

\begin{table}[hbp]
\centering
\small
\resizebox{\textwidth}{!}{%
\begin{tabular}{p{2.0cm}p{2.4cm}cccc}
\toprule
Dataset & Primary metric & \CE & \SOL & \WSOL & \(\WSOL-\CE\) \\
\midrule
NAB & Standard profile & \(18.186 \pm 2.186\) & \(18.626 \pm 3.360\) & \(19.791 \pm 2.781\) & \(+1.605 \pm 2.730\) \\
SKAB & Standard profile & \(43.592 \pm 3.152\) & \(46.657 \pm 2.290\) & \(47.170 \pm 2.905\) & \(+3.578 \pm 2.317\) \\
Exathlon & AD4 & \(0.4200 \pm 0.0375\) & \(0.2708 \pm 0.0398\) & \(0.2615 \pm 0.0427\) & \(-0.1585 \pm 0.0394\) \\
\bottomrule
\end{tabular}
}
\caption{Validation-selected held-out comparisons on NAB, SKAB, and
Exathlon.  Values are means \(\pm\) standard errors over 20 split--seed
comparisons after selecting the corresponding loss configuration by validation
performance within each comparison.  The primary metrics are the NAB standard
profile, the SKAB standard profile, and Exathlon AD4; the \(\WSOL-\CE\) column
reports the paired test-score difference across the same selected
split-seed comparisons.}
\label{tab:headline}
\end{table}

SKAB is the cleanest positive setting.  Each series contains a localized changepoint, and the scoring profile rewards alarms in a short window around that change.  This is close to the mathematical setting of Section~\ref{sec:wsol}: pointwise labels
identify event regions, while the evaluation score gives graded temporal credit.

The validation-selected \WSOL\ mean score reaches 47.170, compared with 43.592 for
\CE\ and 46.657 for \SOL.  
This value is not obtained from one fixed temporal weight vector used everywhere: for
each split--seed comparison, the \WSOL\ configuration is selected by validation
performance and then evaluated on the held-out test split.  As a diagnostic, we
also examine what happens if the same window and weights are kept fixed while only
the fold and random seed change.  Under this stricter fixed-candidate view, the
strongest \WSOL\ family is the balanced-accuracy, max-aggregated SKAB-utility
weight vector with \(H=8\).  Using that same \(H=8\) weighting scheme across all 20
comparisons gives a lower mean test score, 46.336, but it is still above
\CE\ in 75\% of the runs.
Longer horizons remain positive on average against \CE\ but give
smaller or less stable gains.  The result is consistent with the intended use of
\WSOL: the useful weighting is local, and the incremental gain over \SOL\ is
present but more modest than the gain over cross-entropy.

\begin{figure}[t]
\centering
\includegraphics[width=0.95\textwidth]{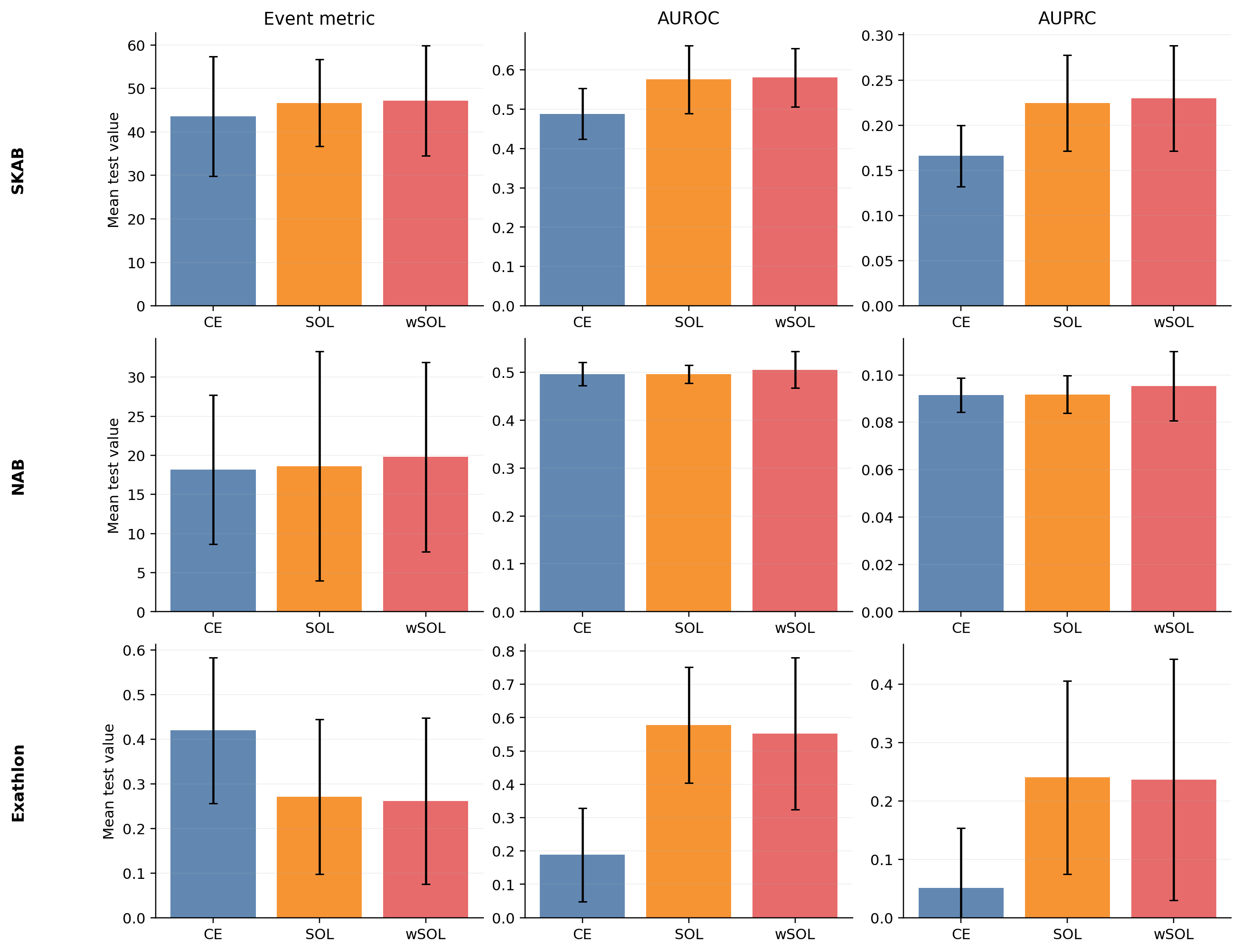}
\caption{Validation-selected held-out comparisons on SKAB, NAB, and
Exathlon.  Rows correspond to datasets, while columns report the event metric,
AUROC, and AUPRC.}
\label{fig:event-pointwise}
\end{figure}

In Figure~\ref{fig:event-pointwise}, the SKAB row shows that the mean gain is not
confined to the event utility.  \WSOL\ also improves the area under the receiver operating characteristic curve (AUROC) and the area under the precision--recall curve (AUPRC) over \CE\ and slightly over \SOL.  This
makes SKAB the cleanest positive case in the study: unlike NAB, where most of
the gain is concentrated on the event utility, and unlike Exathlon, where
ranking can improve while event utility worsens, SKAB shows aligned
improvements in both the deployment score and the pointwise ranking statistics.

Figure~\ref{fig:utility-alignment} makes this interpretation explicit. On SKAB,
the balanced-accuracy family peaks at the shortest tested horizon and then
weakens as the temporal horizon extends.  The same figure also shows the best
fixed temporal weight vectors across datasets: SKAB selects a short localized profile, NAB a
broader graded profile, and Exathlon a short front-loaded profile that still
fails to recover \CE\ on the target event metric. This comparison is not meant
to check whether the selected \(\omega\) exactly duplicates the benchmark
scorer at every time step. Instead, a well-matched surrogate should place weight
on the same side of the event, over a similar temporal span, and with a similar
increase or decay pattern.

\begin{figure}[t]
\centering
\includegraphics[width=0.95\textwidth]{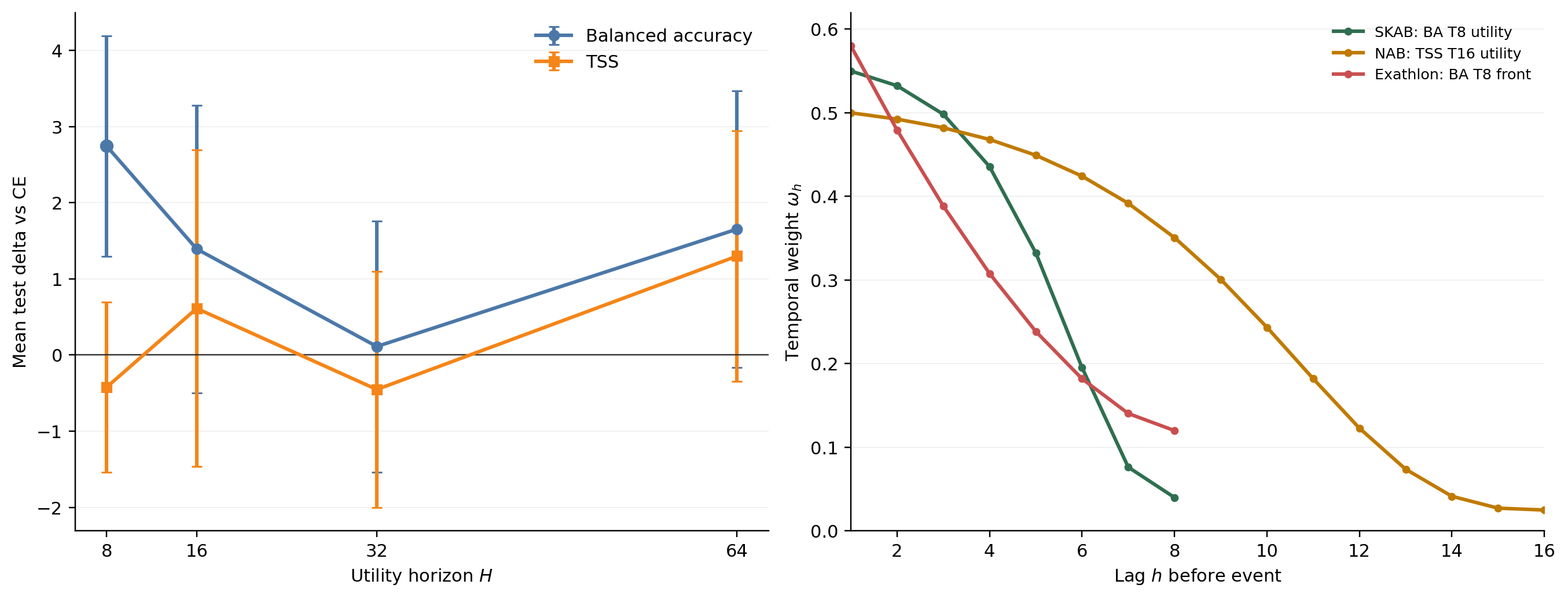}
\caption{Utility alignment in the selected candidates.  Left: SKAB candidate
landscape, shown as mean test delta versus \CE\ across utility horizons; the
balanced-accuracy family peaks at \(H=8\), while longer horizons give smaller
gains.  Error bars show standard errors across split-seed rows.  Right: best
fixed candidate temporal weight vectors across datasets.  SKAB selects a short localized
profile, NAB a longer NAB-shaped graded profile, and Exathlon a short
front-loaded heuristic profile that still does not recover \CE\ on AD4.  The
comparison is qualitative: these temporal weight vectors are surrogates for the benchmark
utilities, not exact copies of the benchmark scorers.}
\label{fig:utility-alignment}
\end{figure}

NAB gives a second positive case, but the result is less uniform.  
In Table~\ref{tab:headline}, the mean evaluation metric is 19.791 for \WSOL, 18.626 for \SOL, and 18.186 for
\CE. The highest mean test score among
the \WSOL\ candidate families is obtained by the TSS-based max-rule \WSOL\ with
\(H=16\). This configuration reaches 22.610 and is above \CE\ in 12 of 20 runs.

The NAB row of Figure~\ref{fig:event-pointwise} compares the event score with
pointwise AUROC and AUPRC.  The \WSOL\ gain is clearer on the event utility than
on the ranking metrics.  This supports the central mechanism: temporal
weighting changes how the probability sequence interacts with thresholding and
post-processing, rather than simply improving global pointwise discrimination.

\begin{figure}[t]
\centering
\includegraphics[width=0.95\textwidth]{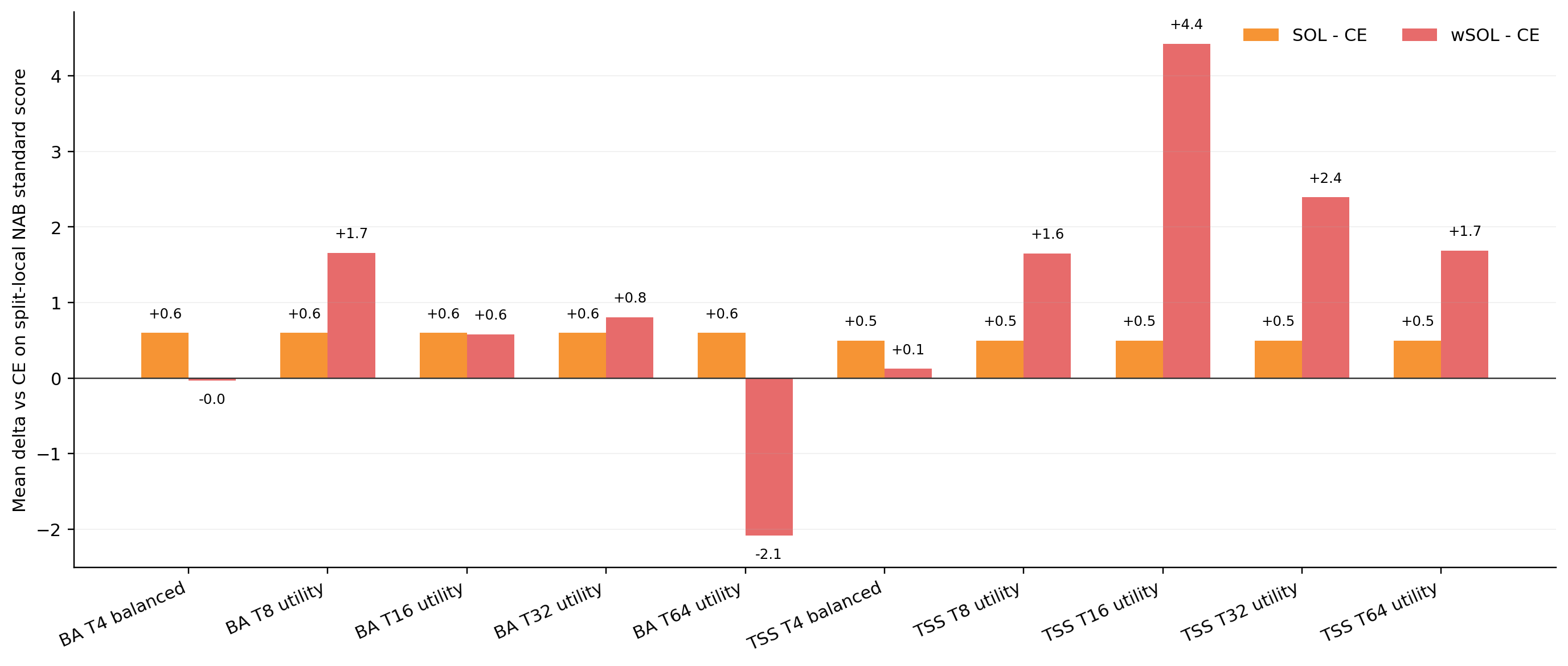}
\caption{NAB candidate-family comparison.  The bars report mean split-local NAB
standard-score deltas relative to \CE\ for the fixed candidate families.  The
\WSOL\ advantage is not uniform across temporal weight vectors: it is strongest
for the longer NAB-shaped horizons, especially under the TSS-based internal
score.  The \(T4\) ``balanced'' candidates are short-horizon controls from the
earlier sweep, whereas the \(T8\)--\(T64\) candidates use NAB-shaped utility
weights.}
\label{fig:nab-candidates}
\end{figure}

Figure~\ref{fig:nab-candidates} shows that temporal weighting is useful only when
the candidate family matches the benchmark profile.  Short or mismatched temporal weight vectors
can be neutral or negative.  This behavior is important for interpretation:
\WSOL\ is not merely a stronger optimizer; it is a way to insert a specific
utility model into the training objective.

Exathlon is an anomaly-detection benchmark, but its labels are structurally
different from NAB and SKAB.  Disturbance annotations cover extended ranges from
root-cause onset through downstream effects.  In this setting, a localized
temporal weight vector can be misaligned with AD4, the strictest Exathlon range score
used here, because AD4 rewards front-positioned range overlap without duplicate
credit for repeated detections of the same interval.

The validation-selected AD4 score is 0.4200 for \CE, 0.2708 for \SOL, and 0.2615
for \WSOL.  Even an oracle diagnostic that selects the best \WSOL\ candidate by
held-out AD4 remains below the
validation-selected \CE\ baseline.  At the same time, the Exathlon row of
Figure~\ref{fig:event-pointwise} shows that \SOL\ and \WSOL\ improve mean AUROC and
AUPRC relative to \CE.  This behavior is consistent with the design of \SOL\ and
\WSOL, which are constructed to optimize confusion-matrix-based metrics such as
BA and TSS.  These metrics are suitable for imbalanced classification and can
improve pointwise rankings, as reflected by AUROC and AUPRC.  However, better
pointwise ranking does not necessarily imply better event-level performance
under the AD4 evaluation protocol.

\section{Discussion}
\label{sec:discussion}

The results' pattern supports the theoretical hypothesis.  \WSOL\ helps when two
conditions hold: the metric rewards localized temporal behavior, and the pointwise
labels do not already provide that behavior as a dense training signal.  SKAB and
NAB satisfy these conditions and give positive mean differences.  Exathlon violates the
locality condition because anomalous labels span long disturbance ranges, and the
weighted objective does not recover the \CE\ baseline on the range-level score.

The experiments also clarify the role of \SOL.  Some benefit comes from replacing
cross-entropy with a score-oriented objective, independent of temporal weights.
This is expected in anomaly-detection benchmarks, where positive labels are rare
and plain cross-entropy can be dominated by the majority normal class.  A
score-oriented loss instead optimizes differentiable or piecewise differentiable
surrogates of skill scores,
which are better aligned with imbalanced classification than accuracy-like
pointwise metrics alone.
For this reason, \SOL\ is a necessary baseline.  On SKAB, \SOL\ already improves
substantially over \CE, while \WSOL\ adds a smaller further gain.  On NAB, the
weighted component is more visible in the best TSS-shaped families.  On Exathlon,
both score-oriented losses improve ranking metrics but not the selected event
metric.

The selected temporal weight vectors are themselves interpretable results, as shown in
Figure~\ref{fig:utility-alignment}.  In \WSOL, the candidate weights are not
learned network parameters but explicit hypotheses about deployment utility.
The relevant question is not whether the selected \(\omega\) is identical to
the benchmark utility function pointwise, because \WSOL\ inserts \(\omega\)
through the surrogate terms \eqref{eq:wfp}--\eqref{eq:wfn}.  Agreement is
structural: the selected temporal weight vector should match the temporal
neighborhood that receives credit, the time scale over which credit is spread, and the
abruptness or gradualness of the decay.

Under this criterion, the positive cases are coherent.  On SKAB, the strongest
fixed family uses a short, front-loaded max temporal weight vector, which agrees with a local
changepoint reward window concentrated near the event.  This is
consistent with \eqref{eq:wfp}--\eqref{eq:wfn}: larger early \(\omega_h\)
reduce the effective false-positive penalty through \(1-a_i^+\) and increase
the false-negative correction \(c_i\) only for nearby alarms, so a narrow
reward window should favor short horizons.  On NAB, the strongest family shifts
to a longer NAB-shaped TSS temporal weight vector, which agrees with the
broader, graded official reward profile over an anomaly window.  The agreement
is stronger than qualitative resemblance.  The best fixed NAB family is the
TSS-shaped temporal weight vector with \(H=16\), and its weights satisfy
\(\omega_h=0.025+0.475\,u_h^3\), with \(u_h\) defined from the official NAB
scaled sigmoid in \eqref{eq:u_h}.  In other words, the winning temporal weight
vector has the same shape as the positive-reward part of the official NAB
utility curve, after rescaling and applying a power transform.  The sweep
therefore selected the candidate explicitly designed to match the utility shape,
not merely a generically decreasing weight vector.
Exathlon is different.  Its
AD4 objective is a front-biased range-overlap F-score defined on extended anomaly
intervals, not a localized one-sided reward window.  A short front-loaded temporal weight vector can still
be the best candidate within the tested \WSOL\ family, but it does not agree
with the full structure of the target utility.

Relative to Marchetti et al. \cite{Marchetti2024}, the theoretical novelty of
the present work is therefore not the introduction of weighted \SOL\ itself, but
its specialization, reformulation, and empirical assessment in the event-based
setting of temporally localized reward windows.

Several limitations remain.  First, the experiments keep the neural backbone
fixed within each dataset.  This design isolates the effect of the loss function:
\CE, \SOL, and \WSOL\ are compared under the same architecture and data pipeline.
The results should therefore be interpreted as controlled loss-level comparisons,
not as claims of state-of-the-art performance on the benchmarks.

Second, the choice of temporal weight vector is important.  In \WSOL, this vector is not
just a nuisance hyperparameter like batch size or dropout.  It encodes an
assumption about the deployment utility: which side of an event should be
rewarded, how wide the useful detection window is, and how quickly the reward
should increase or decay.  Validation can select a useful candidate from the
predefined families, but it may identify only a plausible utility shape rather than a
unique optimum.

Third, the present formulation represents event utility through weighted
confusion-matrix terms.  This captures temporal preferences inside differentiable
or piecewise differentiable skill-score surrogates, but it does not reproduce every detail of a full
benchmark scorer.  Future work could integrate richer event-level constraints,
repeated-alarm penalties, or differentiable approximations of complete benchmark
scoring rules.

\section{Conclusion}
\label{sec:conclusion}

This paper introduced a temporally localized specialization of the weighted
\SOL\ framework for anomaly-detection applications.  The central goal is to make
training sensitive to the timing of an alarm, not only to its pointwise
correctness.  \WSOL\ does this by applying a differentiable skill score to
expected confusion-matrix entries whose false-positive and false-negative terms
are modified by temporal weights through differentiable or piecewise
differentiable operations.  These weights encode the event utility
assumed by the task: alarms close to an event can be treated as less harmful than
ordinary false positives, and a missed event can be penalized less when the model
has already produced a sufficiently strong prior warning.  In this way, the loss
brings part of the thresholded event metric into the training objective.

The experiments show that this mechanism is useful when the benchmark utility
has a localized temporal structure.  On SKAB and NAB, where the target scores
reward alarms within specific time windows, \WSOL\ attains higher mean
validation-selected held-out utility than the matched cross-entropy baseline.  These gains show
that temporal weighting can supply information that ordinary pointwise labels do
not fully provide.  The Exathlon result gives the complementary lesson: \WSOL\
does not improve the range-level AD4 score, even though it improves mean pointwise
ranking metrics.  This separates better probability ranking from better event-level
utility and shows that the temporal weight vector must match the structure of the final
metric.

The practical conclusion is therefore application-dependent.  \WSOL\ should be
used when the deployment objective has a clear temporal reward window and that
window is not already encoded densely in the labels.  When anomalies are already
represented as long labeled ranges, or when the event scorer rewards a different
structure from the chosen temporal weight vector, the weighted objective may not improve the
final utility score over cross-entropy.  Its potential lies in settings where the
cost of an alarm depends strongly on its timing: in such cases, \WSOL\ provides a
principled way to encode this temporal utility directly into neural-network
training. 

\section*{Acknowledgements}

The authors acknowledge the support of GNCS -- National Group for Scientific
Computing, National Institute of Higher Mathematics ``Francesco Severi''
(INdAM).

\section*{Data availability}

The benchmark datasets analyzed in this study are publicly available from the
sources cited in the manuscript. The implementation is available at
\url{https://github.com/edoardolegnaro/ScoreOrientedLosses.git}.

\bibliographystyle{model1-num-names}
\bibliography{references}

\end{document}